\documentclass[conference]{IEEEtran}
\IEEEoverridecommandlockouts
\usepackage{cite}
\usepackage{amsmath,amssymb,amsfonts}
\usepackage{algorithmic}
\usepackage{graphicx}
\usepackage{textcomp}
\usepackage{algorithm}
\usepackage{algorithmic}
\usepackage{graphicx}
\usepackage{subfig}
\usepackage{amsthm}
\usepackage{url}
\usepackage{engord}

\newtheorem{theorem}{Theorem}
\newtheorem{lemma}{Lemma}
\newtheorem{corollary}{Corollary}[theorem]

\def\BibTeX{{\rm B\kern-.05em{\sc i\kern-.025em b}\kern-.08em
    T\kern-.1667em\lower.7ex\hbox{E}\kern-.125emX}}
\newcommand{\sign}{\text{sign}}
\begin{document}

\title{Robust Principal Component Analysis Using a Novel Kernel Related with the $L_1$-Norm

\thanks{This work is funded by the NSF under award 1739396. and it is also supported in part by an award from the University of Illinois at Chicago Discovery Partners Institute Seed Funding Program. The work of D. Badawi is partially supported by NSF under award 1934915. The work of E. Koyuncu is supported in part by the NSF Award CCF-1814717.}
}

\author{\IEEEauthorblockN{Hongyi Pan, Diaa Badawi, Erdem Koyuncu, A. Enis Cetin}
	\IEEEauthorblockA{Department of Electrical and Computer Engineering,
	University of Illinois at Chicago\\
	\{hpan21, dbadaw2, ekoyuncu, aecyy\}@uic.edu}
}
\maketitle


\begin{abstract}
We consider a family of vector dot products that can be implemented using sign changes and addition operations only. The dot products are energy-efficient as they avoid the multiplication operation entirely. Moreover, the dot products induce the $\ell_1$-norm, thus providing robustness to impulsive noise. First, we analytically prove that the dot products yield symmetric, positive semi-definite generalized covariance matrices, thus enabling principal component analysis (PCA). Moreover, the generalized covariance matrices can be constructed in an Energy EFficient (EEF) manner due to the multiplication-free property of the underlying vector products. We present image reconstruction examples in which our EEF PCA method result in the highest peak signal-to-noise ratios compared to the ordinary $\ell_2$-PCA and the recursive $\ell_1$-PCA.


\end{abstract}

\section{Introduction}
In data analysis problems with a large number of input variables, dimension reduction methods are very useful to reduce the size of the input by decreasing the complexity of the problem while sacrificing negligible accuracy. Principal Component Analysis (PCA) and related methods are widely used in data analysis field as dimension reduction techniques \cite{On lines and planes of closest fit to systems of points in space, Nonlinear component analysis as a kernel eigenvalue problem, A generalization of principal component analysis, A tutorial on principal component analysis}.
In most problems, the lower dimensional subspaces that are obtained using the eigenvectors effectively capture the nature of the input data structure. As a result, PCA can be also used in a variety of applications including novelty detection \cite{Kernel PCA for novelty detection, Model selection of Gaussian kernel PCA for novelty detection}, data clustering \cite{K-means clustering via principal component analysis, On the minimum average distortion of quantizers with index-dependent distortion measures, A systematic distributed quantizer design method with an application to MIMO broadcast channels, A semi-NMF-PCA unified framework for data clustering, Turning big data into tiny data, PCA-based high-dimensional noisy data clustering via control of decision errors, A source coding perspective on node deployment in two-tier networks}, denoising \cite{Denoising based on time-shift PCA, MRI noise estimation and denoising using non-local PCA, Diffusion weighted image denoising using overcomplete local PCA, Noise level estimation for model selection in kernel PCA denoising} and outlier detection \cite{Cluster pca for outliers detection in high-dimensional data, Functional outlier detection with robust functional principal component analysis, Image outlier detection and feature extraction via L1-norm-based 2D probabilistic PCA, Robust and Computationally-Efficient Anomaly Detection Using Powers-Of-Two Networks}.

Although the conventional PCA based on the regular dot-product and the $\ell_2$-norm has successfully solved many problems, it is sensitive to outliers in data because the effects of the outliers are not suppressed by the $\ell_2$-norm. It turns out that $\ell_1$-PCA is more robust to outliers and it can be iteratively solved  in $ {O}(N^{rK-K+1})$  for  $D$ dimensional vectors, where $N$ is the number of data vectors,  $ 1\leq K<r=$ (rank of the $N\times D$ data matrix) \cite{Optimal algorithms for l1-subspace signal processing}. Therefore, researchers proposed iterative methods to compute $\ell_1$-PCA to achieve robustness against outliers in data
\cite{Optimal algorithms for l1-subspace signal processing, Principal component analysis based on L1-norm maximization}. 
The recursive $\ell_1$-PCA method requires some parameters to be properly adjusted. On the other hand, the proposed kernel based approach does not need any hyperparameters to be adjusted. This is because we construct a sample covariance matrix using the kernel and obtain the eigenvalues and eigenvectors to define the orthogonal linear transformation instead of solving an optimization problem. 

We recently introduced a family of operators related with $\ell_1$-norm to extract features from image regions and to design Additive neural Networks (AddNet) in a wide range of computer vision applications \cite{An energy efficient additive neural network, Non-euclidean vector product for neural networks, Additive neural network for forest fire detection, Image description using a multiplier-less operator}. We call the new family of operators Energy-Efficient (EEF) operators because they do not require any multiplications which consume more energy compared to additions and binary operations in most processors. 
Instead of a multiplication, the operators use the sign of multiplication and either sum the absolute values of operands, or calculate the minimum or maximum of operands. When we construct dot-product like operations from the EEF operators they induce the $\ell_1$-norm. Details of the EEF-operator are provided in Section \ref{Multiplication-Free Vector Product}. 

In this paper, we define three multiplication-free dot products and construct the corresponding multiplication-free covariance matrices. The fact that the underlying dot product is not an ordinary Euclidean inner product implies that the covariance matrix is not necessarily symmetric and positive semi-definitive. Nevertheless, we analytically prove that two of our vector products yield symmetric and positive semi-definite covariances. Correspondingly, we find the eigenvalues and eigenvectors of the matrices as in regular $\ell_2$-PCA. 
The resulting eigenvectors are orthogonal to each other and one can perform orthogonal projection onto the subspace formed by the eigenvectors to reduce the dimension, perform denoising and other similar PCA applications used in data analysis. In addition, the dot products defined by the operators can be computed without performing any multiplications. Consequently, the matrices of the new kernels can be computed in an energy efficient manner because the new kernels are based on sign operations, binary operations and additions. 

\section{Energy-Efficient (EEF) Vector Products}\label{Multiplication-Free Vector Product}
In this section, we motivate and introduce the family of multiplication-free dot products and establish their relationship to the $\ell_1$-norm.
\subsection{Motivation}
 Let $\mathbf{w} = [w_1 \cdots w_n]^T \in \mathbb{R}^{D\times 1}$ and $\mathbf{x} = [x_1 \cdots x_n]^T \in \mathbb{R}^{D\times 1}$  be two $D$-dimensional column vectors. The standard Euclidean inner product is defined as
\begin{align}
\label{ell2innerproduct}
     \langle \mathbf{w}, \mathbf{x} \rangle = \mathbf{w}^T \mathbf{x} \triangleq \sum_{i=1}^D w_i x_i
\end{align}
Note that because the product $\langle \cdot, \cdot \rangle$ induces the $\ell_2$-norm in the sense that for any $\mathbf{x}$, we have $\langle \mathbf{x}, \mathbf{x}\rangle  = \|\mathbf{x}\|^2 = \sum_{i=1}^D |x_i|^2$.

The $D$ multiplication operations that appear in the inner product Eq. (\ref{ell2innerproduct}) may be costly in terms of energy consumption and time. The existence of multiplications are also undesirable in the presence of outliers: For example, if a component is an outlier with a relatively large magnitude, multiplication will further amplify its effect, making the result of the inner product unreliable. In this context, it has been  recently observed that in many applications, $\ell_1$-based methods outperform $\ell_2$-based methods thanks to their better resilience against outliers or impulse-type noise. These observations motivate us to define the new dot products that induce the $\ell_1$-norm. The new dot products should avoid multiplications both for the sake of computational and energy efficiency as well as robustness. 

\subsection{Multiplication-Free (MF) Dot Products}
In this work, we will evaluate the performance of three different MF operators, described in what follows. Given a real number $a\in\mathbb{R}$, let
\begin{equation}
\mathrm{sign}(a) = 
\begin{cases}
-1,& a<0,\\
0,& a=0,\\
1,& a>0,
\end{cases}
\end{equation}
denote the sign of $a$. Unlike \cite{Additive neural network for forest fire detection} where we define $\sign(0)=1$ or $\sign(0)=-1$ to take advantage of bit-wise operations, we utilize the standard signum function for better precision here. 

First, we introduce our original MF dot product \cite{An energy efficient additive neural network, Non-euclidean vector product for neural networks}. It is defined as
\begin{align}
\mathbf{w}^T\oplus_{mf} \mathbf{x} = \sum_{i=1}^D \sign(w_i x_i)(|w_i| + |x_i|)
\label{eqn_mf_vector}
\end{align}

Note that the only multiplication operations that appears in Eq. (\ref{eqn_mf_vector}) correspond to sign changes and can be implemented with very low complexity. For this reason, we do not count the sign changes towards multiplication operations and thus call Eq. (\ref{eqn_mf_vector}) an MF dot product. It can easily be verified that the product in Eq. (\ref{eqn_mf_vector}) induces a scaled version of $\ell_1$-norm as 
\begin{equation}
	\mathbf{x}^T \oplus_{mf} \mathbf{x} = \sum_{i=1}^n |x_i|+|x_i| = 2\|\mathbf{x}\|_1
\end{equation}

Notice that the original MF dot product conducts scale of 2, we are seeking another $\ell_1$-norm based method without any scaling. We then define a min-based MF dot product:
\begin{align}
	\label{odotproduct}
	\mathbf{w}^T \odot \mathbf{x}  & \triangleq \sum_{i=1}^D \sign(w_i x_i) \min(|w_i|, |x_i|).
\end{align}
and its variation:
\begin{align}
	\label{odotproduct2}
	\mathbf{w}^T\odot_m \mathbf{x} \triangleq \sum_{i=1}^D \mathbf{1}\left(\sign(w_i) = \sign(x_i)\right) \min(|w_i|, |x_i|)
\end{align}
Here, $\mathbf{1}(\cdot)$ is the indicator function. The variant is related to the XX similarity measure \cite{On the positive semi-definite property of similarity matrices}. In Eq. (\ref{odotproduct2}), components of opposite sign $\sign(w_i) \neq \sign(x_i)$ have no contribution towards the dot product, while in Eq. (\ref{odotproduct}), they contribute as a subtractive term. Both of them induce $\ell_1$-norm as 
\begin{equation}
	\mathbf{x}^T \odot \mathbf{x} = \sum_{i=1}^n \min(|x_i|, |x_i|) = \|\mathbf{x}\|_1
\end{equation}

\begin{equation}
	\mathbf{x}^T \odot_m \mathbf{x} = \sum_{i=1}^n \min(|x_i|, |x_i|) = \|\mathbf{x}\|_1
\end{equation}

Vector dot products described above can be extended to matrix multiplications as follows: Let $\mathbf{W} \in \mathbb{R}^{n\times m}$ and $\mathbf{X} \in \mathbb{R}^{n\times p}$ be arbitrary matrices. We then define
\begin{align}
	\mathbf{W^T} \!\oplus\! \mathbf{X} \triangleq 
	\begin{bmatrix}
		\mathbf{w}_1^T\oplus \mathbf{x}_1&\mathbf{w}_1^T\oplus \mathbf{x}_2&\dots&\mathbf{w}_1^T\oplus \mathbf{x}_p\!\!\!\!\!\!\\
		\mathbf{w}_2^T\oplus \mathbf{x}_1&\mathbf{w}_2^T\oplus \mathbf{x}_2&\dots&\mathbf{w}_2^T\oplus \mathbf{x}_p\!\!\!\!\!\!\\
		\vdots&\vdots&\ddots&\vdots&\\
		\mathbf{w}_m^T\oplus \mathbf{x}_1&\mathbf{w}_m^T\oplus \mathbf{x}_2&\dots&\mathbf{w}_m^T\oplus \mathbf{x}_p\!\!\!\!\!\!
	\end{bmatrix}
	\label{eq: matrix}
\end{align}
where  $\oplus\in\{\oplus_{mf}, \odot, \odot_m\}$, $\mathbf{w}_i$ is the $i$th column of $\mathbf{W}$ for $ i = 1,\ 2,\ \dots,\ m$ and $\mathbf{x}_j$ is the $j$th column of $\mathbf{X}$ for $j = 1,\ 2,\ \dots,\ p$.  In brief, the definition is similar to the matrix production $\mathbf{W}^T\mathbf{X}$ by only changing the element-wise product to element-wise MF-operation or element-wise min-operation.

\section{Robust Principal Component Analysis}
\label{Robust Principal Component Analysis}

Suppose that we collect members of a $D$-dimensional dataset $\{\mathbf{x}_1,\ldots,\mathbf{x}_N\}$ to a $D\times N$ matrix $\mathbf{X}=[\mathbf{x}_1 \ \mathbf{x}_2 \ ... \  \mathbf{x}_N]\in \mathbb{R}^{D\times N}$. The well-known $\ell_2$-PCA method relies on investigating the eigendecomposition of the sample covariance matrix 
\begin{align}
\label{l2samplecovariance}
\mathbf{C} = \mathbf{X}\mathbf{X}^T.
\end{align}
We have omitted normalization by the number of elements $N$ of the dataset as it will not change the final eigenvectors and the order of eigenvalues. Elementary linear algebra guarantees that $\mathbf{C}$ has non-negative eigenvalues (i.e. $\mathbf{C}$ is positive semi-definite) and thus the eigenvector corresponding to the $i$th largest eigenvalue becomes the $i$th principal vector. 

In this work, we propose to investigate the 
analogue of Eq. (\ref{l2samplecovariance}) for MF operators. In other words, we consider the eigendecomposition of \begin{align}
\label{l1samplecovariance}
\mathbf{A} = \mathbf{X} \oplus \mathbf{X}^T,
\end{align}
where $\oplus\in\{\oplus_{mf},\odot,\odot_m\}$. Matrix $\mathbf{A}$ is called as MF-covariance matrix. Note that the ordinary matrix product in Eq. (\ref{l2samplecovariance}) is replaced by the MF product in Eq. (\ref{l1samplecovariance}). On the other hand, since $\mathbf{A}$ is no longer constructed using $\ell_2$-products, it is not guaranteed to be symmetric or positive semi-definite. Still, we have the following result.
\begin{theorem}
\label{theorem1}
Let $\oplus\in\{\odot,\odot_m\}$. Then, $\mathbf{A} = \mathbf{X} \oplus \mathbf{X}^T$ is symmetric and positive semi-definite for any $\mathbf{X}$.
\end{theorem}
The proof can be found in the appendix. In particular, the theorem shows that   $\odot$ and $\odot_m$ describe Mercer-type kernels. Theorem \ref{theorem1} paves the way for extending PCA to multiplication-free operators $\odot$ and $\odot_m$, as shown via Algorithm \ref{al: l1pca}.

\begin{algorithm}[h]
	\caption{Algorithm for $L_1$ PCA using MF operators}
	\begin{algorithmic}[1]
		\renewcommand{\algorithmicrequire}{\textbf{Input:}}
		\renewcommand{\algorithmicensure}{\textbf{Output:}}
		\REQUIRE $\mathbf{X}=[\mathbf{x}_1 \ \mathbf{x}_2 \ ... \  \mathbf{x}_N]\in \mathbb{R}^{D\times N}$
		\ENSURE  $\mathbf{W} \in \mathbb{R}^{D\times K}$
		\STATE Construct the MF covariance matrix $\mathbf{A}$ of $\mathbf{X}$ based on Eq. (\ref{l1samplecovariance}). 
		\STATE $[\mathbf{W}, \mathbf{D}] = \text{eigs}(\mathbf{A}, K)$
		\RETURN $\mathbf{W}$.\\
		Comment: Step 2 represents eigendecomposion of {\bf A} and returns a subset of diagonal matrix $\mathbf{D}$ of $K$ largest eigenvalues and matrix $\mathbf{W}$ whose columns are the corresponding right eigenvectors, so that $\mathbf{AW} = \mathbf{WD}$.
		Compared with the conventional $L_2$-PCA Algorithm, we can see that the only difference is at Step 1. We replace the standard covariance matrix by the multiplication-free covariance matrix.
	\end{algorithmic}
	\label{al: l1pca}
\end{algorithm}

The conclusions of Theorem \ref{theorem1} does not hold for the $\oplus_{mf}$ operator. A counterexample is provided by the dataset $\mathbf{x}_1 = [1\,\,2]^T,\,\mathbf{x}_2 = [-1\,-2]^T$, which yields a generalized covariance matrix $\mathbf{A} = [\begin{smallmatrix} 2 & 6 \\ 6 & 8 \end{smallmatrix}]$ with a negative determinant, and thus not positive semi-definite.

\section{Experimental Results}
In this section, we carry out an image reconstruction and denoising experiment using the EEF kernel based PCAs, $\ell_2$-PCA  and the recursive $\ell_1$-PCA to illustrate the robustness of the EEF kernel introduced in Section \ref{Robust Principal Component Analysis}. Image reconstruction example is the same as the experiment in \cite{Optimal algorithms for l1-subspace signal processing}.  The source code of \cite{Optimal algorithms for l1-subspace signal processing} is available in \cite{L1-PCA Toolbox}, so we only set the tolerance parameter of the recursive $\ell_1$-PCA method as $1\times10^{-8}$ as suggested by the author P. Markopoulos. For convenience, we name our method based on Eq. (\ref{eqn_mf_vector}) as "MF-$\ell_1$"-PCA, method based on Eq. (\ref{odotproduct}) as "min-$\ell_1$-PCA-1" and method based on Eq. (\ref{odotproduct2}) as "min-$\ell_1$-PCA-2", respectively, in Table \ref{tab: image reconstruction} and Table \ref{tab: image reconstruction_sp}.

In 
the first row of Fig. \ref{fig: image reconstruction}, we have three $128\times128=16384$ "clean" gray-scaled images ($\mathbf{I}\in \{0, \frac{1}{255}, ...,  \frac{255}{255}\}^{128\times128}$). We assume that the image $\mathbf{I}$ is not available but we have $N=10$ occluded versions $\mathbf{I}_1, \mathbf{I}_2, ..., \mathbf{I}_{10}$, are available as shown in the second row of Fig. \ref{fig: Statue 1} and Fig. \ref{fig: Cat}. The occluded images are created by partitioning the original image $\mathbf{I}$ into sixteen tiles of size $32\times32$ and replacing three arbitrarily selected tiles by $32\times32$ gray-scale-noise patches. The noise patches are in the uniformly random distribution in the interval $(0, 1)$. 

In the second experiment, we add salt and pepper noise to images and restore the original images using various PCA methods. We assume that the image $\mathbf{I}$ is not available but we have $N=10$ corrupted versions $\mathbf{I}_1, \mathbf{I}_2, ..., \mathbf{I}_{10}$, are available as shown in the third column (Fig. \ref{fig: Pikachu}) and the forth column (Fig. \ref{fig: Car}) of Fig. \ref{fig: image reconstruction}, respectively. The corrupted images are created by adding salt and pepper noise to the original image $\mathbf{I}$ with noise density 0.1. In other words, this affects 10\% pixels by making them either $0$ or $1$ assuming that the image pixel values are in the range of $[0, 1]$.

\begin{figure}[htb!]
    \centering
    \subfloat[Statue 1]{
		\begin{minipage}{0.23\linewidth}
			\includegraphics[width=1\linewidth]{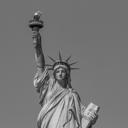}\vspace{3pt}
			\includegraphics[width=1\linewidth]{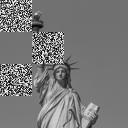}\vspace{3pt}
			\includegraphics[width=1\linewidth]{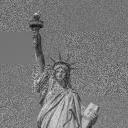}\vspace{3pt}
			\includegraphics[width=1\linewidth]{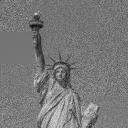}\vspace{3pt}
			\includegraphics[width=1\linewidth]{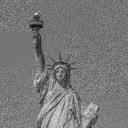}\vspace{3pt}
			\includegraphics[width=1\linewidth]{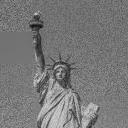}\vspace{3pt}
			\includegraphics[width=1\linewidth]{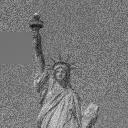}
			\label{fig: Statue 1}
	\end{minipage}}
	\subfloat[Cat]{
		\begin{minipage}{0.23\linewidth}
			\includegraphics[width=1\linewidth]{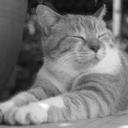}\vspace{3pt}
			\includegraphics[width=1\linewidth]{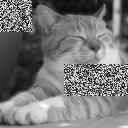}\vspace{3pt}
			\includegraphics[width=1\linewidth]{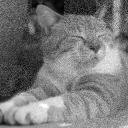}\vspace{3pt}
			\includegraphics[width=1\linewidth]{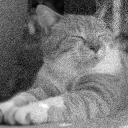}\vspace{3pt}
			\includegraphics[width=1\linewidth]{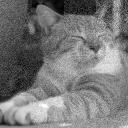}\vspace{3pt}
			\includegraphics[width=1\linewidth]{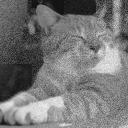}\vspace{3pt}
			\includegraphics[width=1\linewidth]{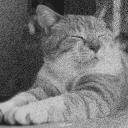}
			\label{fig: Cat}
	\end{minipage}}
	\subfloat[Pikachu]{
		\begin{minipage}{0.23\linewidth}
			\includegraphics[width=1\linewidth]{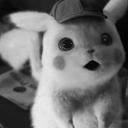}\vspace{3pt}
			\includegraphics[width=1\linewidth]{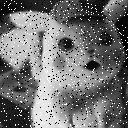}\vspace{3pt}
			\includegraphics[width=1\linewidth]{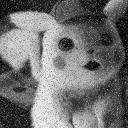}\vspace{3pt}
			\includegraphics[width=1\linewidth]{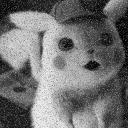}\vspace{3pt}
			\includegraphics[width=1\linewidth]{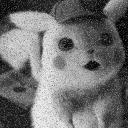}\vspace{3pt}
			\includegraphics[width=1\linewidth]{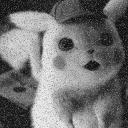}\vspace{3pt}
			\includegraphics[width=1\linewidth]{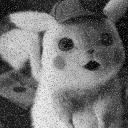}
			\label{fig: Pikachu}
	\end{minipage}}
	\subfloat[Car]{
		\begin{minipage}{0.23\linewidth}
			\includegraphics[width=1\linewidth]{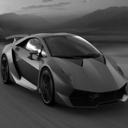}\vspace{3pt}
			\includegraphics[width=1\linewidth]{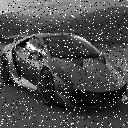}\vspace{3pt}
			\includegraphics[width=1\linewidth]{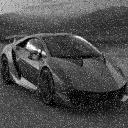}\vspace{3pt}
			\includegraphics[width=1\linewidth]{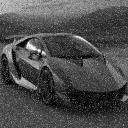}\vspace{3pt}
			\includegraphics[width=1\linewidth]{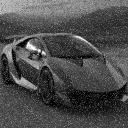}\vspace{3pt}
			\includegraphics[width=1\linewidth]{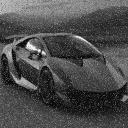}\vspace{3pt}
			\includegraphics[width=1\linewidth]{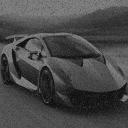}
			\label{fig: Car}
	\end{minipage}}
    \caption{Samples of image reconstruction results. Images in each columns are ordered as the original image (\engordnumber{1} row), the noise patches occluded image (\engordnumber{2} row, \engordnumber{1} and \engordnumber{2} columns) or salt-and-pepper noise corrupted image (\engordnumber{2} row, \engordnumber{3} and \engordnumber{4} columns), results of $\ell_2$-PCA (\engordnumber{3} row),  recursive $\ell_1$-PCA (\engordnumber{4} row), MF-$\ell_1$-PCA (\engordnumber{5} row), min-$\ell_1$-PCA-1 (\engordnumber{6} row) and min-$\ell_1$-PCA-2 (\engordnumber{7} row), respectively.}
    \label{fig: image reconstruction}
\end{figure}

\begin{table*}[htbp]
	\centering
	\caption{PSNR (dB) of Image Reconstruction Results of Noise Patches}
	\begin{tabular}{lllllll}
		\hline\noalign{\smallskip}
		\textbf{Images}&\textbf{Noisy Image}&\textbf{$L_2$-PCA}&\textbf{Recursive $L_1$-PCA\cite{Optimal algorithms for l1-subspace signal processing}}&\textbf{MF-$L_1$-PCA}&\textbf{Min-$L_1$-PCA-1}&\textbf{Min-$L_1$-PCA-2}\\
		\noalign{\smallskip}\hline\noalign{\smallskip}
        Lenna&16.7629&22.1631&24.7089&24.8434&24.7535&\textbf{26.0864}\\
        Statue 1&17.6746&26.6841&26.8580&27.0726&\textbf{28.1158}&27.6244\\
        Statue 2&16.8727&24.6482&25.0187&\textbf{25.0195}&24.9855&24.9843\\
        Earth&14.9133&22.1300&21.7362&21.8622&22.3728&\textbf{23.5439}\\
        Pikachu&15.2871&18.8648&22.7063&22.7091&22.8233&\textbf{23.3173}\\
        Flower&16.3968&21.2211&24.4147&24.4805&\textbf{24.6941}&24.6062\\
        Orange&15.7715&23.5939&23.6927&23.7039&24.4659&\textbf{25.1322}\\
        Cat&16.9120&24.7980&24.8240&\textbf{24.8706}&24.6083&24.6245\\
        Food&15.9369&22.7220&23.8487&23.8629&23.7767&\textbf{24.0548}\\
        Car&15.4178&23.1472&23.3687&23.2191&\textbf{23.6331}&23.6324\\
        Cobra&16.8129&22.3284&25.0985&25.1166&\textbf{25.2137}&24.7121\\
        River&17.2655&24.5775&25.2324&\textbf{25.4168}&24.9636&24.9636\\
        Butterfly&16.6592&24.3993&24.9096&24.8675&24.9390&\textbf{27.2313}\\
        Bridge&15.6619&22.2344&22.9112&\textbf{23.0056}&22.7037&22.7037\\
        \noalign{\smallskip}\hline\noalign{\smallskip}
        Average&16.3104&23.1080&24.2378&24.2893&24.4321&\textbf{24.8012}\\
		\noalign{\smallskip}\hline
	\end{tabular}
	\label{tab: image reconstruction}
\end{table*}

\begin{table*}[htbp]
	\centering
	\caption{PSNR (dB) of Image Reconstruction Results of Salt and Pepper Noise}
	\begin{tabular}{lllllll}
		\hline\noalign{\smallskip}
		\textbf{Images}&\textbf{Noisy Image}&\textbf{$L_2$-PCA}&\textbf{Recursive $L_1$-PCA\cite{Optimal algorithms for l1-subspace signal processing}}&\textbf{MF-$L_1$-PCA}&\textbf{Min-$L_1$-PCA-1}&\textbf{Min-$L_1$-PCA-2}\\
		\noalign{\smallskip}\hline\noalign{\smallskip}
        Lenna&15.4380&24.5958&24.9956&25.0011&\textbf{27.4015}&26.8035\\
        Statue 1&15.9694&25.6774&25.7857&25.7881&\textbf{28.4256}&27.3701\\
        Statue 2&15.6564&24.8093&25.3118&25.3079&\textbf{27.3624}&26.7524\\
        Earth&13.9172&19.8308&22.6834&22.6611&\textbf{24.0921}&\textbf{24.0921}\\
        Pikachu&15.0674&23.8277&24.1352&24.1562&\textbf{24.8679}&24.6878\\
        Flower&15.5086&24.4801&24.9658&24.9929&26.6930&\textbf{27.0417}\\
        Orange&14.5181&21.8783&23.5725&23.5578&23.5971&\textbf{26.3117}\\
        Cat&15.4272&24.5837&24.7230&24.7408&\textbf{25.9535}&25.5513\\
        Food&15.1661&24.2062&24.3647&24.3428&\textbf{25.2575}&24.6552\\
        Car&14.9950&23.6471&24.1611&24.1658&\textbf{26.0886}&\textbf{26.0886}\\
        Cobra&15.5983&20.8556&25.1478&25.1371&\textbf{26.4361}&25.9094\\
        River&15.5659&24.8450&25.1635&25.2141&\textbf{26.8476}&25.8904\\
        Butterfly&14.9599&23.4908&24.4593&24.4333&\textbf{25.0181}&24.8584\\
        Bridge&14.7112&22.7575&23.7324&23.7341&\textbf{24.2897}&24.0655\\
        \noalign{\smallskip}\hline\noalign{\smallskip}
        Average&15.1785&23.5347&24.5144&24.5166&\textbf{25.8808}&25.7199\\
		\noalign{\smallskip}\hline
	\end{tabular}
	\label{tab: image reconstruction_sp}
\end{table*}

We perform PCA on the set of $\mathbf{V} = [\mathbf{v}_1\  \mathbf{v}_2\ ...\ \mathbf{v}_{10}]$, where $\mathbf{v}_i = \text{vec}(\mathbf{I}_i), i=1, 2, ..., 10$, is the vector form of $\mathbf{I}_i$. In this way, we obtain the eigenvector matrix $\mathbf{W}\in \mathbb{R}^{16384\times2}$ of the covariance or the MF-covariance matrices of $(\mathbf{V}-\mathbf{\bar{v}})$. Then, we recover the image $\mathbf{I}$ as
\begin{equation}
\mathbf{\hat{v}}_i=\mathbf{WW}^T(\mathbf{v}_i-\mathbf{\bar{v}})+\mathbf{\bar{v}}
\label{eq: m_hat_0_mean}
\end{equation}
\begin{equation}
\mathbf{\hat{I}} = \text{mat}\mathbf{\hat{(v})_i}
\end{equation}
where $\mathbf{\bar{v}}\in \mathbb{R}^{16384\times 1}$ is the mean value of $[\mathbf{v}_1\  \mathbf{v}_2\ ...\ \mathbf{v}_{10}]$, $\mathbf{0.5}$ or $\mathbf{0}$, $\mathbf{I}_i$ is an arbitrary occluded image, and $\text{mat}(\cdot)$ is the inverse transform of $\text{vec}(\cdot)$ that reshapes a vector back to the matrix form. We calculate $\mathbf{\hat{v}}_i$ in the method that returns the largest peak signal-to-noise-ratio (PSNR). 

PSNR between the reconstructed image $\mathbf{\hat{I}}$ and the original image $\mathbf{I}$ as the following equations is used for evaluation in Table \ref{tab: image reconstruction} and Table \ref{tab: image reconstruction_sp}:
\begin{equation}
\text{MSE} = \text{mean}((\mathbf{\hat{I}}-\mathbf{I})^2),
\end{equation}
\begin{equation}
\text{PSNR} = 10\text{log}_{10}(\frac{\text{peakval}^2}{\text{MSE}}),
\end{equation}
where $(\cdot)^2$ is the element-wise square and ``peakval" is the peak signal value. The higher the value of PSNR is, the better the reconstruction result is.

Our experiment is summarized in Algorithm \ref{al: image reconstruction experiment}. Results of these PCA methods are shown in Fig. \ref{fig: image reconstruction} for four test images and their statistics are provided in Table \ref{tab: image reconstruction} and Table \ref{tab: image reconstruction_sp}. Although which method works the best depends on the images, our three methods return larger PSNR than $\ell_2$-PCA and the recursive $\ell_1$-PCA in both experiments, and the two min-$\ell_1$-PCAs are better than the MF-$\ell_1$-PCA, globally.
For example, the min-$\ell_1$-PCA produces about $1.4dB$ better than the recursive $\ell_1$-PCA in the salt-and-pepper noise removal experiment.

\begin{algorithm}[htbp]
	\caption{Image Reconstruction Experiment}
	\begin{algorithmic}[1]
		\renewcommand{\algorithmicrequire}{\textbf{Input:}}
		\renewcommand{\algorithmicensure}{\textbf{Output:}}
		\REQUIRE N corrupted images $\mathbf{I}_1, \mathbf{I}_2, ..., \mathbf{I}_N\in \mathbb{R}^{D\times D}$.
		\ENSURE  Reconstructed image $\mathbf{\hat{I}}$.
		\FOR{$i=1, 2, ..., N$} 
		    \STATE $\mathbf{v_i} = \text{vec}(\mathbf{I_i})\in \mathbb{R}^{D^2\times 1}$;
		\ENDFOR 
		\STATE $\mathbf{V} = [\mathbf{v}_1\ \mathbf{v}_2\ ...\  \mathbf{v}_N]\in \mathbb{R}^{D^2\times N}$;
		\STATE $\mathbf{\bar{v}} = \mathbf{0}, \mathbf{0.5}$ or $\text{mean}(\textbf{V})\in \mathbb{R}^{D^2\times 1}$;
		\STATE Run PCA on $(\mathbf{V}-\mathbf{\bar{v}})$ to obtain $K$-dominant eigenvector matrix $\mathbf{W} = [\mathbf{w}_1\ \mathbf{w}_2 \ ... \  \mathbf{w}_K]\in \mathbb{R}^{D^2\times K}$;
		\STATE $\mathbf{\hat{v}}_i=\mathbf{WW}^T(\mathbf{v}_i-\mathbf{\bar{v}})+\mathbf{\bar{v}}\in \mathbb{R}^{D^2\times 1}$;
		\STATE $\mathbf{\hat{I}} = \text{mat}(\mathbf{\hat{v}}_i) \in \mathbb{R}^{D\times D}$;
		\RETURN $\mathbf{\hat{I}}$.\\
		Comment: In this experiment, $N=10, D=128$ and $K=2$. Function $\text{mean}(\cdot)$ is the mean of each row, so it returns a column vector. Function $\text{vec}(\cdot)$ reshapes a matrix into the column vector form, and function $\text{mat}(\cdot)$ is its inverse transform that reshapes a column vector back to the matrix form. $(\mathbf{V}-\mathbf{\bar{v}})$ is defined as $[\mathbf{v}_1-\mathbf{\bar{v}}\ \mathbf{v}_2-\mathbf{\bar{v}}\ ...\  \mathbf{v}_N-\mathbf{\bar{v}}]$.
	\end{algorithmic}
	\label{al: image reconstruction experiment}
\end{algorithm}


We also compared the computational cost of the PCA algorithms 
 to reconstruct an image in MATLAB. As it is shown in Table \ref{tab: time}, $\ell_2$-PCA is the fastest algorithm, while our proposed kernel methods are slightly slower than $\ell_2$-PCA but significantly faster than the recursive $\ell_1$-PCA. The recursive $\ell_1$-PCA is the slowest because it obtains the result by recursion, while $\ell_2$-PCA and our three methods return the result straight-forwardly. The reason why our kernel PCAs run a little slower than $\ell_2$-PCA is that, the time to construct an MF-covariance matrix is slightly slower compared to the sample covariance matrix, which is optimized in MATLAB. The computational cost of eigenvalue-eigenvector computations are the same in both $\ell_2$-PCA and the proposed kernel-PCAs.
 

\begin{table}[htbp]
    \centering
	\caption{Computational cost in seconds}
    \begin{tabular}{llll}
    \hline\noalign{\smallskip}
    \textbf{Image Size}&\textbf{$L_2$-PCA}&\textbf{Recursive $L_1$-PCA\cite{Optimal algorithms for l1-subspace signal processing}}&\textbf{Our PCAs$^a$}\\
    \noalign{\smallskip}\hline\noalign{\smallskip}
    $32\times32$&0.02&3.57&0.02\\
    $48\times48$&0.07&4.80&0.09\\
    $64\times64$&0.21&6.50&0.25\\
    $80\times80$&0.49&8.04&0.50\\
    $96\times96$&1.03&10.81&1.10\\
    $112\times112$&1.81&14.02&1.91\\
    $128\times128$&3.24&20.38&3.38\\
    \noalign{\smallskip}\hline\noalign{\smallskip}
    \multicolumn{4}{l}{\textbf{$^a$} Proposed kernel PCAs are comparable to the regular PCA.}\\
    \multicolumn{4}{l}{Due to space limitation, we list them in one column.}
			\vspace{-10pt}
    \end{tabular}
    \label{tab: time}
\end{table}

\section{Conclusion}
In this paper, we proposed three new robust PCA methods. We have reached the following conclusions: (i) Proposed novel kernel methods are more energy-efficient than $\ell_2$-PCA because their Gram matrices are computed without any multiplication operations. (ii) They do not suffer from outliers in the data as in $\ell_2$-PCA because they are based on the $\ell_1$-norm. (iii) They do no require any hyper-parameter optimization as in the recursive $\ell_1$-PCA \cite{Optimal algorithms for l1-subspace signal processing} because their Gram matrices are straightforward to compute as described in Eq. (\ref{l1samplecovariance}). 

We compared the new kernel-based methods with the $\ell_2$-PCA and the recursive $\ell_1$-PCA on an image reconstruction and salt-and-pepper noise removal tasks and found out that our min-$\ell_1$-PCAs returns the largest PSNR among these methods in most scenarios.


\appendix
Let $\mathbf{x}, \mathbf{y} \in \mathbb{R}^N$. We define the min-operator $\oplus : \mathbb{R}^N \times\mathbb{R}^N  \mapsto \mathbb{R}$ as following
\begin{equation}
    \mathbf{x} \oplus \mathbf{y} := \sum_{i=1}^{N} \text{sgn}(x_iy_i) \min(|x_i|, |y_i|)
\end{equation}
In the following we will show that the operator $\oplus$ defines a valid kernel $K(\mathbf{x},\mathbf{y})$.
A symmetric function $K: \mathbb{R}^N \times\mathbb{R}^N  \mapsto \mathbb{R}$ is a kernel iff 
\begin{equation}\label{eqn:kernel_defn}
    \sum_{i=1}^{N} \sum_{j=1}^{N} a_i a_j K(\mathbf{x_i}, \mathbf{x_j) \geq 0}
\end{equation}
for any reals $a_i, a_j$ and for any vectors $\mathbf{x_i}, \mathbf{x_j} \in \mathbb{R}^N$. In our case, we are interested in proving that $K(\mathbf{x_i}, \mathbf{x_j}) = \mathbf{x}_i \oplus \mathbf{x}_j$ satisfies Eq. \ref{eqn:kernel_defn}.

Define a matrix $\mathbf{K} \in \mathbb{R}^{N\times N}$ such that $\mathbf{K}_{ij} = \text{sgn}(x_i x_j)\text{min}(|x_i|,|x_j|)$. Proving that $K(.,.)$ is a valid kernel is equivalent to proving that the matrix $\mathbf{K}$ is positive semi-definite.

We will use the following facts to construct our proof that $\oplus$ is a kernel:
\begin{theorem}[Schur product theorem]\cite{schur1911}
Let $\mathbf{A}, \mathbf{B} \in \mathbb{R}^{N \times N}$ be two positive semi-definite matrices, then their Hadamard product $({\mathbf{A}\odot\mathbf{B}})_{ij}:=\mathbf{A}_{ij}\mathbf{B}_{ij}$ is also positive semi-definite.
\end{theorem}
\begin{lemma}\cite{On the positive semi-definite property of similarity matrices} R. Nader, A. Bretto, B. Mourad and H. Abbas. ``On the Positive Semi-definite Property of Similarity Matrices." \emph{Theoretical Computer Science}
Let $\mathbf{x} \in \mathbf{R}^N$ be a strictly positive vector. Then the matrix $\mathbf{A}_{ij}:=\min(x_i,x_j)$ is positive semi-definite.
\end{lemma}
Our claim is the following:
\begin{corollary}
Let $\mathbf{x} \in \mathbf{R}^N$. Then the matrix $\mathbf{K}_{ij}:= \text{sgn}(x_i x_j)\min(|x_i|, |x_j|)$ is positive semi-definite.
\end{corollary}
\begin{proof}
The matrix $\mathbf{K}_{ij}$ can be written as hadamard product between matrix $\mathbf{B}_{ij}=\text{sgn}(x_i)\text{sgn}(x_j)$ and $\mathbf{A}_{ij}=\min(|x_i|,|x_j|)$, the matrix $\mathbf{B}$ is a (rank-one) positive semi-definite matrix since it can be written as $\text{sgn}(\mathbf{x})\text{sgn}(\mathbf{x})^T$. The matrix $\mathbf{A}$ is positive semi-defnite according to Lemma 1. The Hadamard product $\mathbf{K}=\mathbf{A}\odot\mathbf{B}$ is positive semi-definite according to Theorem 1. Thus the $\oplus$ operator defines a valid kernel.
\end{proof}
\end{document}